\documentclass[nohyperref]{article}

\usepackage{microtype}
\usepackage{graphicx}
\usepackage{booktabs} 
\usepackage{thmtools} 
\usepackage{thm-restate}

\usepackage{hyperref}

\usepackage[accepted]{icml2022}

\usepackage{amsmath}
\usepackage{amssymb}
\usepackage{mathtools}
\usepackage{amsthm}
\usepackage{multirow}
\usepackage{subcaption}
\usepackage[section]{placeins}

\newcommand{\norm}[1]{\left\| #1 \right\|}

\usepackage{xspace}

\newcommand{\lambdamax}{\lambda_{\max}}

\newcommand{\avein}{\frac{1}{n}\sum_{i=1}^n}

\newcommand{\RR}{\mathbb{R}}

\newcommand{\bI}{\mathbf{I}}
\newcommand{\bA}{\mathbf{A}}

\newcommand{\bC}{\mathbf{C}}

\newcommand{\bL}{\mathbf{L}}
\newcommand{\bM}{\mathbf{M}}
\newcommand{\bP}{\mathbf{P}}

\newcommand{\bW}{\mathbf{W}}

\newcommand{\bLambda}{\mathbf{\Lambda}}
\newcommand{\bSigma}{\mathbf{\Sigma}}

\newcommand{\del}[1]{}

\newcommand{\inp}[2]{\left\langle #1 ,  #2 \right\rangle}

\newcommand{\eqdef}{\stackrel{\text{def}}{=}}

\newcommand{\Prob}{\mathbf{Prob}}

\newcommand{\Exp}[1]{{\color{red}\mathbb{E}}\left[#1\right]}

\DeclareMathOperator*{\argmin}{argmin}

\DeclareMathOperator{\Diag}{Diag}       

\usepackage[colorinlistoftodos,bordercolor=orange,backgroundcolor=orange!20,linecolor=orange,textsize=scriptsize]{todonotes}

\newcommand{\cost}{\text{cost}}

\usepackage[capitalize,noabbrev]{cleveref}
\newtheorem{theorem}{Theorem}

\newtheorem{corollary}{Corollary}
\theoremstyle{plain}
\theoremstyle{definition}
\newtheorem{definition}{Definition}
\newtheorem{assumption}{Assumption}
\newtheorem{proposition}{Proposition}
\theoremstyle{remark}
\newtheorem{remark}{Remark}

\icmltitlerunning{Understanding Progressive Training Through the Framework of RCD}
\graphicspath{{../plots/}}

\begin{document}

\twocolumn[
\icmltitle{Understanding Progressive Training Through the Framework of \\Randomized Coordinate Descent}



\icmlsetsymbol{equal}{*}

\begin{icmlauthorlist}
\icmlauthor{Rafa\l{} Szlendak}{kaust}
\icmlauthor{Elnur Gasanov}{kaust}
\icmlauthor{Peter Richt\'arik}{kaust}
\icmlcorrespondingauthor{Firstname1 Lastname1}{first1.last1@xxx.edu}
\end{icmlauthorlist}

\icmlaffiliation{kaust}{CEMSE, King Abdullah University of Science and Technology, Thuwal, Saudi Arabia}
\icmlkeywords{Progressive Training, IST, RCD}

\vskip 0.3in
]



\printAffiliationsAndNotice{} 
\begin{abstract}
We propose a Randomized Progressive Training algorithm (RPT) -- a stochastic proxy for the well-known Progressive Training method (PT) \citep{progrssiveKarras}.
Originally designed to train GANs \cite{goodfellow2014generative}, PT was proposed as a heuristic, with no convergence analysis even for the simplest objective functions. On the contrary, to the best of our knowledge,
RPT is the first PT-type algorithm with rigorous and sound theoretical guarantees for general smooth objective functions. We cast our method into the established framework of Randomized Coordinate Descent (RCD) \citep{nesterovRCD, richtarikRCD}, for which (as a by-product of our investigations) we also propose a novel, simple
and general convergence analysis encapsulating strongly-convex, convex and nonconvex objectives. We then use this framework to establish a convergence theory for RPT.
Finally, we validate the effectiveness of our method through extensive computational experiments.

\end{abstract}

\section{Introduction}

The modern practice of supervised learning typically relies on training huge-dimensional, overparametrized models using first-order gradient methods. However, in most cases, using the full gradient 
information in the training procedure is computationally infeasible. The gold-standard approach of bypassing this difficulty is provided by the family of Stochastic Gradient Descent (SGD) algorithms which use only a partial information about the gradients for the reward of reduced compute time. Progressive Training (PT) \cite{progrssiveKarras},
albeit a non-stochastic method, relies on a similar idea. In its original form, it was developed as a method for training generative adversarial networks (GANs) \cite{goodfellow2014generative} by progressively growing the neural network that we are training.
Specifically, for a fixed period of time we only train the first layer of our GAN while keeping the weights corresponding to the other layers intact; then, we perform updates with respect to the first two layers, and so on. 

\citet{progrssiveKarras} justified two main advantages of this approach: an increased stability of the training (since the training procedure starts from a simpler, reduced model), and reduced 
computation cost. The method was proposed as a heuristic, with no convergence guarantees even for the simplest objectives, such as quadratics or convex functions.

Observe that in the PT procedure, the domain of the parameters which we are trying to learn is partitioned into blocks which correspond to the respective layers of the neural network that PT aims to train. Motivated by this, in this work we are going to be concerned 
with minimizing the general objective functions $f$ whose domain will be written as $\RR^{d_1}\times\dots\times\RR^{d_B}$. For example, $\RR^{d_i}$ might represent the parameters of the $i^\text{th}$ layer of the neural network in question.

\Cref{alg:MAIN-COMPACT} provides formal statement of the PT procedure for the general differentiable function $f:\RR^{d_1}\times\dots\times\RR^{d_B}\rightarrow \RR$. The training procedure is divided into $B$ epochs; in each epoch $b\in[B]$ we want to update the first $b$ coordinate blocks. Let $G_b f(x)\in\RR^{d_1}\times\dots\times\RR^{d_B}$ denote the vector
$$
\left(\nabla_1 f(x), \dots, \nabla_bf(x), 0, \dots, 0\right),
$$
where $\nabla_i f$ is the gradient of $f$ with respect to the $i^\text{th}$ coordinate block. In each epoch $b$ we make the gradient update 
$$
x \leftarrow x - \gamma_bG_b f(x)
$$
for $T_b$ iterations, where $\gamma_b > 0$ is a suitably chosen stepsize which is fixed for the entire epoch $b$.
\begin{algorithm}
	
	\begin{algorithmic}[1]
	\label{alg:PT}
	\STATE {\bfseries Input:} initial iterate $x \in \RR^{d_1}\times\dots\times\RR^{d_B}$; stepsizes $\gamma_1,\gamma_2,\cdots,\gamma_n > 0$; integers $T_1,T_2,\cdots,T_n\geq 0$
	
	\FOR {$b = 1, \cdots, B$} 
		\FOR {$t =  0, \dots, T_b-1$}
			\STATE $x \leftarrow x - \gamma_bG_b f(x)$\label{line:PTb}
		\ENDFOR 
	\ENDFOR
	\STATE {\bfseries Output:} $x$
	\end{algorithmic}
	\caption{Progressive Training}
	\label{alg:MAIN-COMPACT}
	\end{algorithm}

\subsection{Existing literature}
\citet{wangProgFed} attempt to analyze the federated version of PT called ProgFed. Unfortunately, as argued in Appendix~\ref{sec:wangEtAl}, their analysis is vacuous. 

Recently, a novel and related research thread emerged around the idea of \emph{independent subnet training} (IST) which relies on partitioning the neural network in question into sparser subnets, training them in parallel across the devices and periodically aggregating the model. \citet{YuanSubnetTraining} introduce IST and provide its analysis. However, the authors make rather 
strong assumptions on the objective function, such as Lipschitz continuity of the objective function and certain assumptions on the stochastic gradient error. Analogous technique is applied by \citet{DunResIST} for training ResNet \cite{HeResNet}. It is important to note that the IST framework differs from RPT in several key components.
Indeed, IST partitions the neural network in question into sparser submodels which are then trained across the devices \emph{in parallel}. The model is then aggregated and this procedure repeats. In contrast, RPT performs a \emph{single update at a time} of a random subnetwork \emph{chosen in a very particular way} (which simulates the deterministic PT).
Moreover, IST is a distributed method; in contrast, in our paper we focus on the single-node setup.

Notably, none of the aforementioned papers contains mathematically sound analysis for general smooth functions.

\subsection{Our contributions and organization of the paper}
As we have argued, there is still a considerable gap in theoretical understanding of the PT mechanism. \emph{The key objective of this paper is fixing this gap by proposing a randomized proxy for PT which mimics its behavior and has simple and general analysis with mild assumptions.} The contributions of our work can be summarized as follows:
\begin{itemize}
	\item In \Cref{sec:NewAnalysisOfRCD} we develop a simple and general analysis of Randomized Coordinate Descent (RCD) for strongly-convex, convex and nonconvex objectives. Our analysis recovers previous convergence rates of RCD and allows us to fully characterize its iteration complexity in stronlgy-convex, convex and nonconvex regimes via a single
	quantity $L_\bP$. 
	\item In Sections \ref{sec:RPTintro} and \ref{sec:AnalysisOfRPT} we propose a method called RPT which mimics the behavior of PT, and, as we argue, is the first PT-type method with sound theoretical guarantees for smooth objectives. Instead of progressively expanding the model which we are trying to optimize, in each iteration RPT makes a choice of the submodel according to a carefully chosen randomness, and then makes a gradient update of this model. We then view RPT as a particular instance of RCD, which immediately establishes its convergence. 
	We finally identify regimes in which RPT can be advantageous over vanilla gradient descent (GD) through a novel and general notion of total computation complexity.
	\item In \Cref{sec:Experiments}, we present a series of empirical experiments that demonstrate the advantages of our proposed method. These experiments are divided into three sub-sections. The first sub-section employs synthetic data to demonstrate that, when appropriately generated, our proposed Randomized Progressive Training (RPT) method can converge at a faster rate than Gradient Descent while also incurring less computation cost. In the second and third sub-sections, we utilize real-world datasets to evaluate the performance of the RPT method, including testing on a ridge regression problem across various datasets and on an image classification task on the CIFAR10 dataset. The results demonstrate that the RPT method exhibits comparable or even superior performance in comparison to other methods.
\end{itemize}

\subsection{Assumptions}
In this section we lay out the assumptions that we are going to make throughout the whole work. In an attempt to make the analysis as general as possible, in this paper we focus on the problem of minimizing an arbitrary differentiable function $f:\RR^d\rightarrow\RR$, where $\RR^d$ is decomposed as $\RR^{d_1}\times\dots\times\RR^{d_B}$ for some pre-specified $d_1, \dots, d_B$.

We first recall the notion of \emph{matrix smoothness} \citep{quSDNA, safaryanSmoothnessMatrices} which generalizes the standard $L$-smoothness assumption and allows one to encapsulate more curvature information of the objective function.
\begin{definition}
   Let $f:\RR^d\rightarrow\RR$ and $\bL$ be a positive semidefinite matrix. We say that $f$ is $\bL$-smooth if
   $$
      f(x) - f(y) - \inp{\nabla f(y)}{x - y}\leq \frac{\norm{x - y}_\bL^2}{2}
   $$
   for all $x, y\in\RR^d$, where $\norm{v}_{\bL} \eqdef \sqrt{v^\top \bL v}$. 
\end{definition}
We make the following assumption throughout the whole paper.
\begin{assumption}
   \label{as:matrixSmoothness}
 $f$ is $\bL$-smooth.
\end{assumption}
With \Cref{as:matrixSmoothness} taken for granted, we analyse RPT in each of the three standard settings:
\begin{assumption}
   \label{as:stronglyConvex}
   $f$ is $\mu$-strongly convex, that is, there exists a positive constant $\mu$ such that
   $$
      f(x) - f(y) - \inp{\nabla f(y)}{x - y}\geq \frac{\mu\norm{x - y}^2}{2}
   $$
   for all $x, y\in\RR^d$.
\end{assumption}

\begin{assumption}
   \label{as:convex}
   $f$ is convex and has a global minimizer $x^*$.
\end{assumption}

\begin{assumption}
   \label{as:nonconvex}
   $f$ is nonconvex and lower-bounded by $f^*$.
\end{assumption}
Under assumptions \ref{as:stronglyConvex}, \ref{as:convex} and \ref{as:nonconvex} respectively we will be interesting in minimizing the following quantities:
$$
   \Exp{\norm{x^t - x^*}^2}\text{, } \Exp{f(x^t) - f(x^*)} \text{, } \Exp{\norm{\nabla f(x^t)}^2}.
$$

\section{New Analysis of Randomized Coordinate Descent}
\label{sec:NewAnalysisOfRCD}
\subsection{SkGD}
Randomized Coordinate Descent (RCD) is a celebrated technique of minimizing high-dimensional objective functions $f:\RR^d\rightarrow\RR$, where updating all gradient coordinates is prohibitively 
expensive. In its original form \citep{nesterovRCD, richtarikRCD} RCD used a \emph{biased} randomized estimator of the gradient; in each iteration it chooses a random sub-block of coordinates and performs a gradient step with respect to that sub-block alone.

Recently, \citet{safaryanSmoothnessMatrices} developed an \emph{unbiased} version of RCD (called SkGD) which updates the iterates through \emph{unbiased sketches}. Specifically, in each iteration $t$, SkGD starts from choosing a random set of coordinates $S^t\subset [d]$. Denote $p_j = \Prob(j\in S^t)$ and assume $p_j > 0$. 
Then SkGD performs an update of form $$x^{t+1} = x^t - \gamma\bC^t\nabla f(x^t),$$ where
\begin{equation*}
   \bC^t = \Diag(c_1, \dots, c_d), \text{\ \ \ } 
   c_j = \begin{cases}1/p_j \text{ \ \ if }j\in S^t\\ 
      0 \text{ \ \ \ \ \ \ \ \ otherwise.}\end{cases}
\end{equation*}
It is then readily verifiable that $\bC^t\nabla f(x^t)$ is an unbiased gradient estimator, that is 
$$
\Exp{\bC^t \nabla f(x^t)} = \nabla f(x^t).
$$

In the remainder of this paper we will always denote $$\bP\eqdef \Diag(p_1, \dots, p_d).$$
\subsection{Key quantity: $L_\bP$}
We now define the most important component of our analysis. In the rest of the paper $\lambdamax\left(\bM\right)$ will denote the largest eigenvalue of matrix $\bM$.
\begin{definition}
   Let $\bC$ be an unbiased sketch operator and $\bP$ its associated probability matrix. Denote the smoothness matrix of the objective function $f$ by $\bL$. Then
   $$
      L_\bP \eqdef \lambdamax\left(\bP^{-1/2}\bL\bP^{-1/2}\right).
   $$
\end{definition}
It turns out that this quantity describes the convergence of RCD in \emph{all of the regimes in question} (assumptions \ref{as:stronglyConvex}, \ref{as:convex}, \ref{as:nonconvex}). Therefore, 
by estimating it for a particular choice of sampling distribution $\bP$ one obtains a unified analysis of RCD in the aforementioned regimes. Moreover, in the following lemma we provide sharp estimates for $L_\bP$.
\begin{restatable}{lemma}{keyQuantity}
	\label{lem:key_quantity}
	Let $\bL$ be any positive semidefinite matrix, $\bC$ -- any unbiased sketch operator, and $\bP$ -- its probability matrix. Then \begin{align*}\lambdamax(\Exp{\bC\bL\bC}) \leq L_\bP\leq \lambdamax(\bL)\lambdamax(\bP^{-1}).\end{align*}
\end{restatable}
The first bound will be useful in proving the lemma on the convergence of RCD for nonconvex function. The second one quantifies how large can $L_\bP$ be in terms of the smoothness constant of $f$ and $\bP$. In practice, however, $L_\bP$ can be much smaller.
\subsection{Convergence results}
We begin with a preliminary lemma:
\begin{restatable}{lemma}{generalizedSmoothness}
	\label{lem:generalizedSmoothness}
	Let $\bW = \Diag(w_1, \dots, w_d)$, where $w_1, \dots, w_d$ are any positive entries and assume that $f:\RR^d\rightarrow\RR$ is $\bL$-smooth. Then 
	$$
		\norm{\nabla f(x)}_\bW^2\leq 2A(f(x) - f(x^*)),
	$$
	where $A = \lambdamax\left(\bW^{1/2}\bL\bW^{1/2}\right)$.
\end{restatable}
The following estimate of a second moment of the gradient estimator is then an easy corollary of \Cref{lem:generalizedSmoothness}.
\begin{restatable}{lemma}{secondMoment}
	\label{lem:secondMoment}
	Let $\bC:\RR^d\rightarrow\RR^d$ be an unbiased sketch operator. Let $f:\RR^d\rightarrow\RR$ be differentiable and denote $g(x) = \bC\nabla f(x)$. Then 
	$$
		\Exp{\norm{g(x)}^2} = \norm{\nabla f(x)}_{\bP^{-1}}\leq 2L_{\bP}(f(x) - f^*).
	$$
\end{restatable}	
The convergence result for $f$ satisfying \Cref{as:stronglyConvex} is now an immediate consequence of Corollary A.1 of \cite{gorbunovSGD}, which provides a unified analysis of SGD under strong convexity.
\begin{restatable}{theorem}{stronglyConvex}
	\label{thm:rcdStronglyConvex}
	Let $f:\RR^d\rightarrow \RR$ be $\bL$-smooth and $\mu$-strongly convex. Let $\bC$ be an unbiased sketch with associated probability matrix $\bP$. Then the iterates $x^k$ of SkGD with stepsize $0<\gamma\leq 1/L_\bP$ and update rule $x^{k+1} = x^k - \gamma\bC\nabla f(x^k)$ satisfy 
	$$
		\Exp{\norm{x^t - x^*}^2}\leq (1 - \gamma\mu)^k \norm{x^0 - x^*}^2.
	$$
\end{restatable}
Under assumptions \ref{as:convex} and \ref{as:nonconvex} there also exist unified analyses of SGD (respectively \citet{khaledConv} and \citet{khaledNonConv}) where \Cref{lem:generalizedSmoothness} is applicable. However, as we argue in \Cref{sec:Generic}, applying these generic results to our setting leads to \emph{suboptimal} convergence rates. We instead develop a simple analysis
of SkGD in convex and nonconvex regimes and present it as Theorems \ref{thm:rcdConvex} and \ref{thm:rcdNonconvex} respectively.

\begin{restatable}{theorem}{rcdConvex}
   \label{thm:rcdConvex}
	Let $f:\RR^d\rightarrow \RR$ be $\bL$-smoooth, convex, and lower-bounded. Furthermore, suppose that $f$ attains its minimum at $x^*$. Let $(x^t)_{t =0}^\infty$ be the iterates of the SkGD with unbiased sketch $\bC$ and stepsize $0 < \gamma \leq 1 / (2L_\bP)$. Denote $\bar{x}^T = \frac{1}{T+1}\sum_{t = 0}^Tx^t$. Then
	$$
		\Exp{f(\bar{x}^T) - f(x^*)}\leq \frac{\norm{x^0 - x^*}^2}{\gamma(T+1)}.
	$$
	In particular, for $\gamma = 1/(2L_\bP)$ and $T+1\geq \frac{2L_\bP\norm{x^0 - x^*}^2}{\varepsilon}$, we have that
	$$
		\Exp{f(\bar{x}^T) - f(x^*)}\leq \varepsilon.
	$$
\end{restatable}
\begin{restatable}{theorem}{rcdNonconvex}
   \label{thm:rcdNonconvex}
	Let $f:\RR^d\rightarrow \RR$ be $\bL$-smoooth and lower bounded by $f^*$. Let $(x^t)_{t =0}^\infty$ be the iterates of SkGD with unbiased sketch $\bC$ and stepsize $0 < \gamma \leq 1 / L_\bP$. Denote $\hat{x}^T = \argmin_{t = 0, \dots, T}\norm{\nabla f(x^t)}^2$. Then
	$$
		\Exp{\norm{\nabla f(\hat{x}^T)}^2}\leq \frac{2\delta^0}{\gamma(T+1)},
	$$
	where $\delta^0 = f(x^0) - f^*$. In particular, for $\gamma = 1/L_\bP$ and $T+1\geq \frac{2L_\bP\delta^0}{\varepsilon}$, we have that
	$$
		\Exp{\norm{\nabla f(\hat{x}^T)}^2}\leq \varepsilon.
	$$
\end{restatable}

\section{Randomized Progressive Training (RPT)}
\label{sec:RPTintro}
\subsection{Problem formulation}
\label{subsec:formulation}
For the sake of consistency with the PT framework, we shall decompose the domain of the objective function $f$ into $B$ disjoint subblocks. Thus, we will view the domain of $f$ as $\RR^{d_1}\times\dots\times\RR^{d_B}$, where $\sum_{i = 1}^B d_i = d$. We will further assume that $f$ is $\mathbf{L}$-smooth, where $\bL$ is a symmetric, positive definite matrix of size $d\times d$. We let $\bL_i$ be the $i^\text{th}$ diagonal subblock of $\bL$ 
of size $d_i\times d_i$. We put $L_i\eqdef\lambdamax(\bL_i)$. We will be interested in minimizing this function under assumptions \ref{as:stronglyConvex}, \ref{as:convex} or \ref{as:nonconvex}, i.e. finding
$$
\argmin_{x\in\RR^{d_1}\times\dots\times\RR^{d_B}}f(x).
$$
\subsection{RPT}
We now introduce the \emph{PT-sketch operator}, the key building block of RPT.
\begin{definition}\label{def:PTsketch}
	Let $1 =  p_1\geq  p_2\geq\dots\geq p_B>0$. We additionally put $ p_{B+1} = 0$. Let $\sigma = \left(\sigma_1, \dots, \sigma_B\right)$ be any permutation of set $[B]$. For each $i\in[B]$ define the matrix $\bC_i\in\RR^{d\times d}$ as follows:
	\begin{equation*}
		\bC_i = \Diag\left((\xi^i_{\sigma_1})_{d_1}, \dots, (\xi^i_{\sigma_B})_{d_B}\right),
	\end{equation*}
	where $\Diag(v)$ denotes the diagonal matrix with $v$ as its diagonal, $(x)_m\eqdef (x, \dots, x)\in\RR^m$, and 
	$$
		\xi^i_{j} =
		\begin{cases}
			1/ p_j &\text{ if } j\leq i\\
			0 &\text{ otherwise.}
		\end{cases}
	$$
    We finally define a random matrix $\bC$ which is equal to $\bC_i$ with probability $ p_i -  p_{i+1}$ and call it a \emph{PT-sketch}.
\end{definition}
With this definition, we are now ready to formulate the RPT algorithm.
\begin{algorithm}[H]
   \begin{algorithmic}[1]
   \STATE {\bfseries Input:} initial iterates $x^0\in\RR^{d_1}\times\dots\times\RR^{d_B}$; stepsize $\gamma$; integer $T\geq 0$
   \FOR{$t = 1, \dots, T$}
      \STATE Draw the PT-sketch $\bC^t$
      \STATE $x^{t+1} = x^t - \gamma\bC^t\nabla f(x^t)$
   \ENDFOR
   \STATE {\bfseries Output:} $x^T$
   \caption{Randomized Progressive Training}
   	\label{alg:rpt}
   \end{algorithmic}
\end{algorithm}
RPT is an instance of SkGD with probability matrix $\bP = \Diag\left(\left( p_{\sigma_1}\right)_{d_1}, \dots, \left( p_{\sigma_B}\right)_{d_B}\right)$. Hence, theorems \ref{thm:rcdStronglyConvex}, \ref{thm:rcdConvex}, \ref{thm:rcdNonconvex} immediately apply to it.
Our task is now to choose $\bP$ in a way that minimizes the total computation complexity.

\begin{remark}
	There are two main differences between RPT and classical approaches to progressive training \cite{progrssiveKarras, wangProgFed}. Firstly, our method chooses the sizes of coordinate updates in a randomized fashion.
	PT, however, is a deterministic method which grows the coordinate updates according to a pre-set schedule.

	Moreover, PT assumes a particular order of importance of the coordinate blocks -- it \emph{always} updates the first coordinate block, then includes the second block, and so on. However, we note that in general there is no particular reason 
	to prioritize the first coordinate block, and RPT allows us to decide on the order of importance by choosing the permutation $\sigma$ in \Cref{def:PTsketch}.
\end{remark}
\begin{remark}
	The formulation of RPT is stemming from the problem formulation described in \Cref{subsec:formulation}, which assumes a particular division of the domain elements into $B$ coordinate blocks. While our analysis holds for any such division, we do not
	specify how to decide on it, noting that it is often problem-specific (for instance, a coordinate block might be a layer of a neural network \citep{progrssiveKarras,wangProgFed}).
\end{remark}
\section{Analysis of RPT}
\label{sec:AnalysisOfRPT}

\subsection{Cost model}
In this section we will define the computation cost of optimizing the objective function from \Cref{subsec:formulation}. By construction, in the RPT procedure the variable $x^t\in\RR^{d_1}\times\dots\times\RR^{d_B}$ is divided into $B$ coordinate blocks and updated with respect to a certain subset of these blocks. 
In our discussion we will assume that by carrying out an update with respect to block $i$ the algorithm suffers a penalty $c_i$. Hence, if $\bC^t = \bC_i$, then $\cost_t$, the cost incurred in the iteration $t$ is equal to $\sum_{j = 1}^ic_j$. Moreover, without loss of generality we make a simplifying normalizing assumption on the costs $c_i$ for the sake of presentation.
\begin{assumption}
	\label{as:normalized}
	$\sum_{i = 1}^Bc_i = 1$.
\end{assumption}

Now, recall that $\bC = \bC_i$ with probability $ p_i -  p_{i + 1}$. Therefore, the expected cost of the entire optimization procedure is 
\begin{align*}
	\Exp{\text{cost}(T)} &\eqdef \Exp{\sum_{t = 1}^T\cost_t}\\
	&= T\sum_{i = 1}^B( p_i -  p_{i+1})\sum_{j = 1}^ic_{\sigma_j} \\
	&= T\sum_{i = 1}^B p_ic_{\sigma_i}.
\end{align*} 

Previous convergence results show that in order to reach the $\varepsilon$-accurate solution in all regimes 
described by assumptions \ref{as:stronglyConvex}, \ref{as:convex}, and \ref{as:nonconvex}, one needs $T$ to be proportional to $L_\bP$. Thus,
$$
	\Exp{\text{cost}(T)}\propto c(\bP) \eqdef L_\bP\sum_{i = 1}^B p_ic_{\sigma_i}
$$
regardless of the regime in question. Thus, choosing the optimal RPT sketch is equivalent to solving the minimization problem
\begin{equation}
	\label{eq:complexity}
\min_{\sigma}\min_{1= p_1\geq\dots\geq p_B>0} L_\bP\sum_{i = 1}^B p_ic_{\sigma_i}.
\end{equation}

\subsection{Choosing the probabilities $ p_i$}
Finding the global minimum of the quantity $c(\bP)$ in general is a challenging task when $\bL$ is simply any positive definite matrix; $c(\bP)$ is a non-convex function that can intricately depend on the structure of matrix $\bL$.
However, note that taking $ p_i = 1$ for all $i\in[B]$ recovers gradient descent, i.e. the optimal $c(\bP)$ is at most $L$. Thus, the solution of Problem~(\ref{eq:complexity}) is not worse than gradient descent. Furthermore, one can find a closed-form minimizer for the \emph{tight upper bound} for 
$L_\bP$. We first invoke Lemma 1 from \citet{nesterovRCD}.
\begin{restatable}{lemma}{twoBounds}
	\label{lem:twoBounds}
	Consider an arbitrary positive semidefinite matrix $\bM\in\RR^{d\times d}$ with diagonal subblocks $\bM_1, \dots. \bM_B$, where each $\bM_i\in\RR^{d_i\times d_i}$. Put $M_i = \lambda_{\max}(\bM_i)$ and $M = \lambdamax(\bM)$. Then
	$$
	\max_i M_i \leq M \leq \sum_{i = 1}^B M_i.
	$$
\end{restatable}
In \Cref{sec:tightness} we argue that these bounds are tight.

We then have
$$
c(\bP) \leq \sum_{i = 1}^B\frac{L_{\sigma_i}}{p_i}\cdot\sum_{i = 1}^B p_ic_{\sigma_i}.
$$
The below lemma describes the global minimizer of this upper bound.
\begin{restatable}{lemma}{upperBound}\label{lem:upperBound}
	$$
	\sum_{i = 1}^B\frac{L_{\sigma_i}}{p_i}\cdot\sum_{i = 1}^B p_ic_{\sigma_i}\geq \left(\sum_{i = 1}^B\sqrt{L_ic_i}\right)^2.
	$$
	Moreover, this bound is met by taking $\sigma$ to be such that
	$$
		\sqrt{L_{\sigma_1} / c_{{\sigma_1}}} \geq \dots \geq \sqrt{L_{\sigma_B} / c_{{\sigma_B}}},
	$$
	and
$$
 p_i = \frac{\sqrt{L_{\sigma_i}/c_{\sigma_i}}}{\max_j\sqrt{L_j / c_j}}.
$$

\end{restatable}
This choice of probabilities is similar to the one proposed by \cite{Zhu2016EvenFA} (which, contrary to our paper, treats \emph{accelerated version} of RCD) where the probabilities $ p_i$ are proportional to $\sqrt{L_i}$. However, \cite{Zhu2016EvenFA} consider the setup of non-overlapping blocks and analyse iteration complexity only. 
On the contrary, our work, to the best of our knowledge, is the first one to consider such choice of probabilities in a non-accelerated setting and motivates this choice through the total complexity analysis.

We summarize the discussion with the following corollary:
\begin{corollary}
	Let $\bP$ be the optimal solution of \Cref{eq:complexity}. Then
	$$
		c(\bP) \leq \min\left\{L, \left(\sum_{i = 1}^B\sqrt{L_ic_i}\right)^2\right\}.
	$$
\end{corollary}

\section{Experiments}
\label{sec:Experiments}

We conduct several experiments to evaluate the performance of Randomized Progressive Training in comparison to other methods on various target functions and datasets.

\subsection{Quadratic functions on synthetic data}

In the first experiment, we show how fast Randomized Progressive Training can be compared to vanilla Gradient Descent. When running RPT with probabilities proposed in~\Cref{lem:upperBound}, theory suggests that total computation cost of $O((\sum\limits_{i=1}^B \sqrt{L_i c_i})^2 / \varepsilon)$ is required to achieve $\varepsilon$-accuracy. At the same time, standard theory for GD under~\Cref{as:normalized} (meaning that a cost of one GD iteration is one) suggests $O(L / \varepsilon)$ total cumulative cost. We vary block smoothness constants $L_i$ and computation costs $c_i$ so that the ratio $ (\sum\limits_{i=1}^B \sqrt{L_i c_i})^2  / L$ is relatively large. In this vein, we generate nine quadratic functions $f(x) = \frac12 x^\top \bA x$ with varying dimensions and properties of the \textit{diagonal} matrix $\bA$. The number of blocks is three in all experiments. We consider the computation cost $c_i$ of block $i$ to be proportional to its size. As shown on~\Cref{fig:quad_functions}, when the size of the block with the smallest smoothness constant is much larger than one of the largest $L_i$, RPT shows tremendous speedup over GD. \Cref{table:quad_functions} compares speedups predicted by theory with actual ones. 

\begin{figure*}[h]
\includegraphics[width=\linewidth]{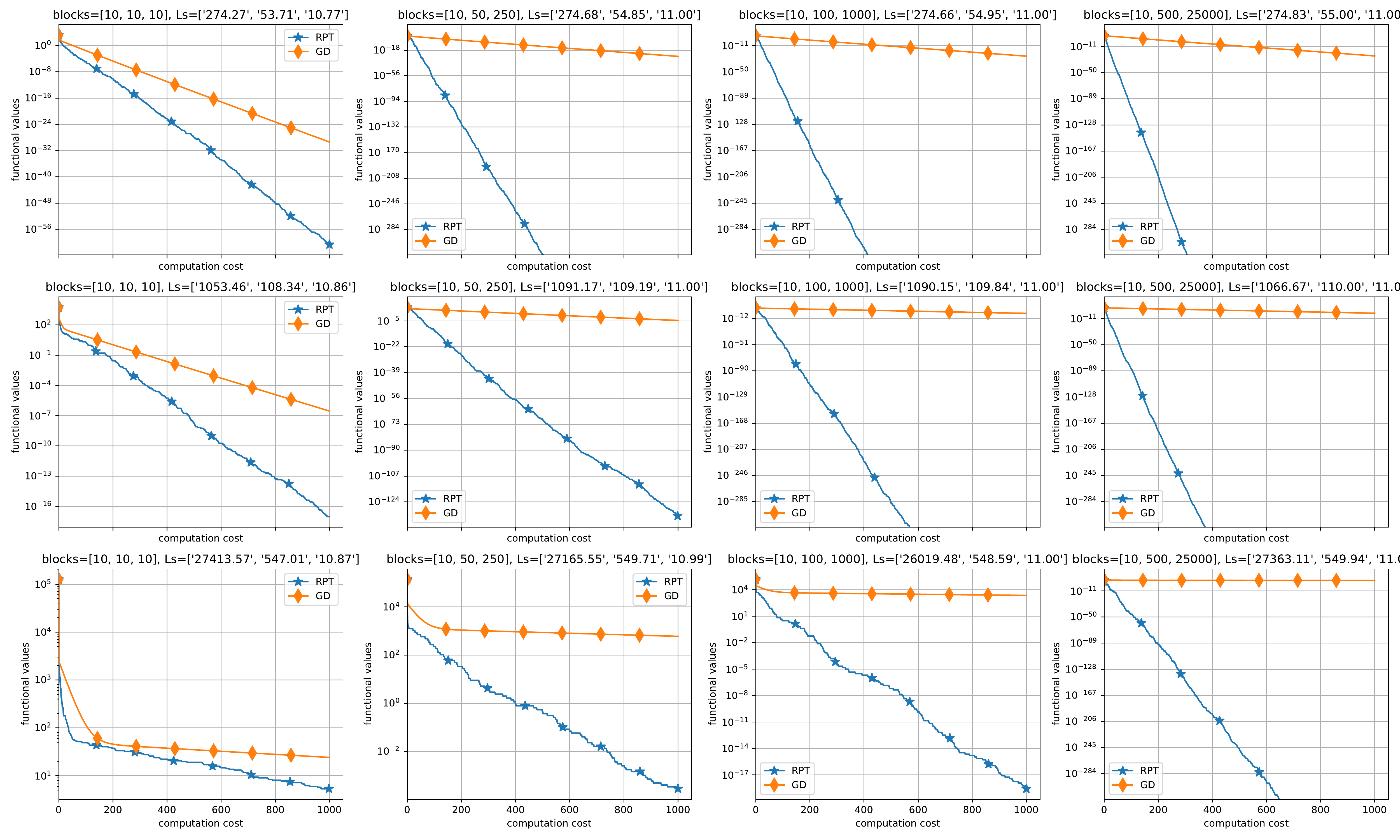}
\caption{Comparison of Randomized Progressive Training (RPT)  Gradient Descent (GD) on quadratic functions. Target function is $f(x)~=~\frac12 x^\top \bA x $, where the \textit{diagonal} matrix $\bA \in \RR^{l\times l}$ and its dimension $l$ differ in each plot. For RPT, the iterate is separated into three blocks. 'blocks' in the plot title indicates the block sizes, 'Ls' indicates smoothness constants within the corresponding block. Numbers in 'blocks' sum up to $l$. The block computation cost $c_i$ is the size of block $i$ divided by $l$, and the computation cost of a single GD iteration is one. Probabilities for RPT are set as stated in~\Cref{lem:upperBound}. }
\label{fig:quad_functions}
\end{figure*}

\begin{table}
	\centering
	\begin{tabular}{|c|c|c|c|c|c|}
		\hline
		\multicolumn{4}{|c|}{Setup} & Theory & Real\\
		\hline
		blocks & 10 & 10 & 10 & \multirow{2}{*}{1.1} & \multirow{2}{*}{1.9}\\
		\cline{1-4}
		Ls & 272 & 53.3 & 11 & & \\
		\hline
		blocks & 10 & 50 & 250 & \multirow{2}{*}{3.4} & \multirow{2}{*}{19.8}\\
		\cline{1-4}
		Ls & 270.5 &  55 & 11 & & \\
		\hline		
		blocks & 10 & 100 &1000 & \multirow{2}{*}{5.7} & \multirow{2}{*}{26.9}\\
		\cline{1-4}
		Ls & 256.7 & 54.8 & 11 & & \\
		\hline
		blocks & 10 & 500 & 25000  & \multirow{2}{*}{12.7} & \multirow{2}{*}{36.6}\\
		\cline{1-4}
		Ls & 274.8 & 55 & 11 & & \\		
		\hline
		blocks & 10 & 10 & 10 & \multirow{2}{*}{1.5} & \multirow{2}{*}{2.2}\\
		\cline{1-4}
		Ls & 1066.5 & 108.4 & 11  & & \\
		\hline
		blocks & 10 & 50 & 250 & \multirow{2}{*}{6.4} & \multirow{2}{*}{17.6}\\
		\cline{1-4}
		Ls & 1080.4 & 109.6 & 11  & & \\
		\hline		
		blocks & 10 & 100 &1000 & \multirow{2}{*}{12.3} & \multirow{2}{*}{64.9}\\
		\cline{1-4}
		Ls & 1093.9 & 109.7 & 11 & & \\
		\hline
		blocks & 10 & 500 & 25000  & \multirow{2}{*}{36.6} & \multirow{2}{*}{128.1}\\
		\cline{1-4}
		Ls &  1066.7 & 110 & 11  & & \\		
		\hline
		blocks & 10 & 10 & 10 & \multirow{2}{*}{2.2} & \multirow{2}{*}{2.8}\\
		\cline{1-4}
		Ls & 27414 & 547 & 11  & & \\
		\hline
		blocks & 10 & 50 & 250 & \multirow{2}{*}{15.4} & \multirow{2}{*}{20}\\
		\cline{1-4}
		Ls & 27166 & 550 & 11  & & \\
		\hline		
		blocks & 10 & 100 &1000 & \multirow{2}{*}{40.1} & \multirow{2}{*}{62.4}\\
		\cline{1-4}
		Ls & 26020 & 549 & 11 & & \\
		\hline
		blocks & 10 & 500 & 25000  & \multirow{2}{*}{282.5} & \multirow{2}{*}{407.8}\\
		\cline{1-4}
		Ls &  27363 & 550 & 11  & & \\		
		\hline		
	\end{tabular}
	\caption{Comparison of theoretical and actual speedup of Randomized Progressive Training versus Gradient Descent in a synthetic experiment using quadratic functions.}
	\label{table:quad_functions}
\end{table}

\subsection{Ridge regression on real datasets}

In our second set of experiments, we assess the convergence of Randomized Progressive Training, Gradient Descent, and Cyclic Block Coordinate Descent~\cite{cbcd} on the ridge regression problem
\begin{align*}
\min f(x) = \avein (\langle a_i, x \rangle - b_i)^2 + \lambda \|x\|^2_2,
\end{align*}
where $a_i \in \RR^d$, $b_i \in \RR$ and $\lambda$ is a regularization parameter. 

We consider four regression datasets: California housing~\cite{kelleyspace}, prostate cancer, Los Angeles ozone~\cite{elements}, and white wine quality~\cite{CORTEZ2009547}. Each dataset feature is normalized, and a bias vector is added. Then, heuristically, features are permuted so that the diagonal of the Hessian matrix is sorted in descending order. As in the previous experiment, we consider $c_i$ proportional to the corresponding block size.

As can be seen from~\Cref{fig:ridge_regression}, although the theoretical speedup of RPT over GD, according to~\Cref{table:ridge_regression}, is close to 1, the actual speedup is still noteworthy. However, both algorithms are slower than Cyclic Block Coordinate Descent. 

\begin{figure*}
	\includegraphics[width=\linewidth]{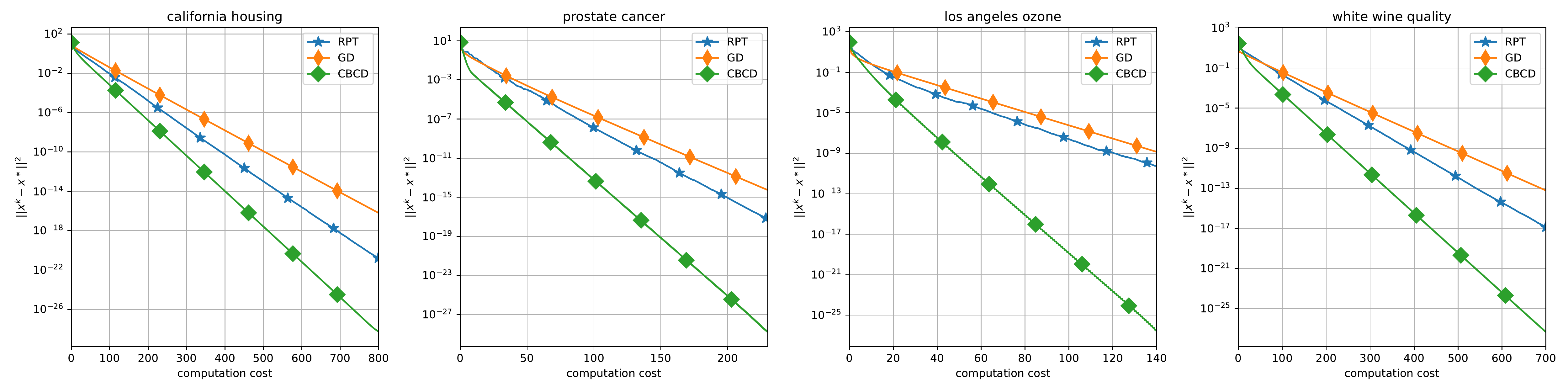}
	\caption{Comparison of Randomized Progressive Training (RPT) vs. Gradient Descent (GD) and Cyclic Block Coordinate Descent (CBCD) on the ridge regression problem. The block computation cost $c_i$ is the size of block $i$ divided by the dimension of the problem, and the computation cost of one GD iteration is 1. Probabilities of block sampling for RPT are set as stated in~\Cref{lem:upperBound}.}
	\label{fig:ridge_regression}
\end{figure*}

\begin{table}
\centering
\begin{tabular}{|c|c|c|}
\hline
\multirow{2}{*}{Dataset} & Theoretical & Actual \\
& speedup & speedup \\
\hline
california housing & 1 & 1.3 \\
\hline
prostate cancer & 1 & 2.1 \\
\hline
los angeles ozone & 1 & 2.4 \\
\hline
white wine quality & 1.1 & 1.3 \\
\hline
\end{tabular}
\caption{Comparison of theoretical and actual speedup of Randomized Progressive Training versus Gradient Descent for ridge regression problems.}
\label{table:ridge_regression}
\end{table}

\subsection{Image classification on CIFAR10}

In our final computational experiment, we trained a Deep Residual Neural Network~\cite{He_2016_CVPR} with 18 layers (ResNet18) for the task of image classification on the CIFAR10 dataset~\cite{Krizhevsky09learningmultiple}. 

{\bfseries Implementation.} The CIFAR10 dataset comprises 50'000 training samples and 10'000 test samples, which are distributed evenly among 10 classes.. The samples are RGB images of size $32\times32$. To facilitate hyperparameter tuning, we further divided the training dataset into training and validation datasets with a ratio $9:1$. During the training process, we applied a padding of 4, randomly cropped to size $32 \times 32$, and randomly flipped images horizontally. We trained the neural network with three algorithms: Stochastic Gradient Descent, Progressive Training~\cite{wangProgFed}, and our own method. We implemented all settings using the Python packages JAX~\cite{jax2018github}, Flax~\cite{flax2020github}, and Optax.

{\bfseries Fine-tuning.} Fine-tuning is a crucial step in the process of training neural networks, as the lack of smoothness and convexity properties inherent in these models often renders the application of theoretical results ineffective~\cite{intrig_prop}. To address this, a common approach for hyperparameter tuning is to utilize a grid search method.  In our case, for the RPT method, we fine-tuned the learning rate and sampling probabilities. A grid search was conducted for the learning rate over the values $[1, 0.5, 0.1, 0.05, 0.01, 0.001]$, and for sampling probabilities, we evaluated two different strategies. The first strategy used exponentially decaying probabilities with rates $[0.999, 0.995, 0.99, 0.95, 0.9, 0.85]$, \textit{e.g.}, for a rate of $0.9$, the probabilities were calculated as $(p_1, \dots, p_{22}) = (1, 0.9, 0.9^2, \dots, 0.99^{21})$. The second strategy used evenly spaced probabilities between $p_1 = 1$ and $p_{22}$, where $p_{22}$ ran over the list $[0.7, 0.6, 0.5, 0.4, 0.3, 0.2]$. The best results were obtained with a learning rate of $1$ and exponentially decaying sampling probabilities with a rate of $0.9$. For Progressive Training (PT), we only fine-tuned the learning rate and utilized the round sizes $T_1, \dots, T_{22}$ suggested in~\cite{wangProgFed} as the best strategy. Specifically, we set $T_1  = \dots = T_{21} = \frac{T}{44}$, and $T_{22} = \frac{23T}{44}$, where $T$ is the total number of steps. The grid search for the learning rate was the same as for RPT. Additionally, Stochastic Gradient Descent parameters such as the learning rate, weight decay, and momentum were set according to the guidelines in~\cite{benchopt}. The batch size in all three settings was 128.
\begin{figure*}
	\begin{subfigure}[c]{0.5\linewidth}
		\centering
		\includegraphics[width=\linewidth]{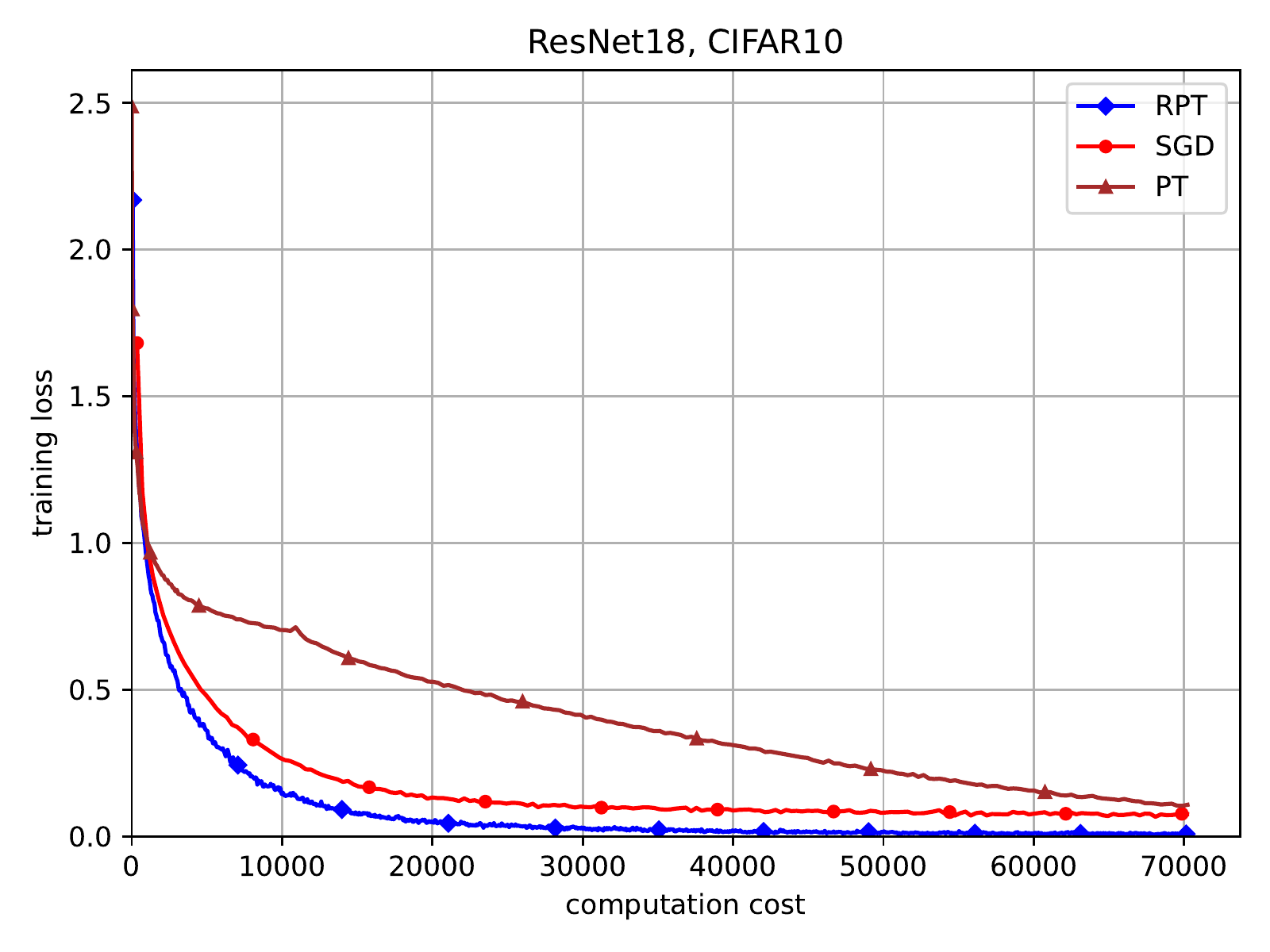}
		\caption{Training loss versus computation cost}
		\label{fig:cifar10_train_cost}
	\end{subfigure}
	\hfill
	\begin{subfigure}[c]{0.5\linewidth}
		\centering		
		\includegraphics[width=\linewidth]{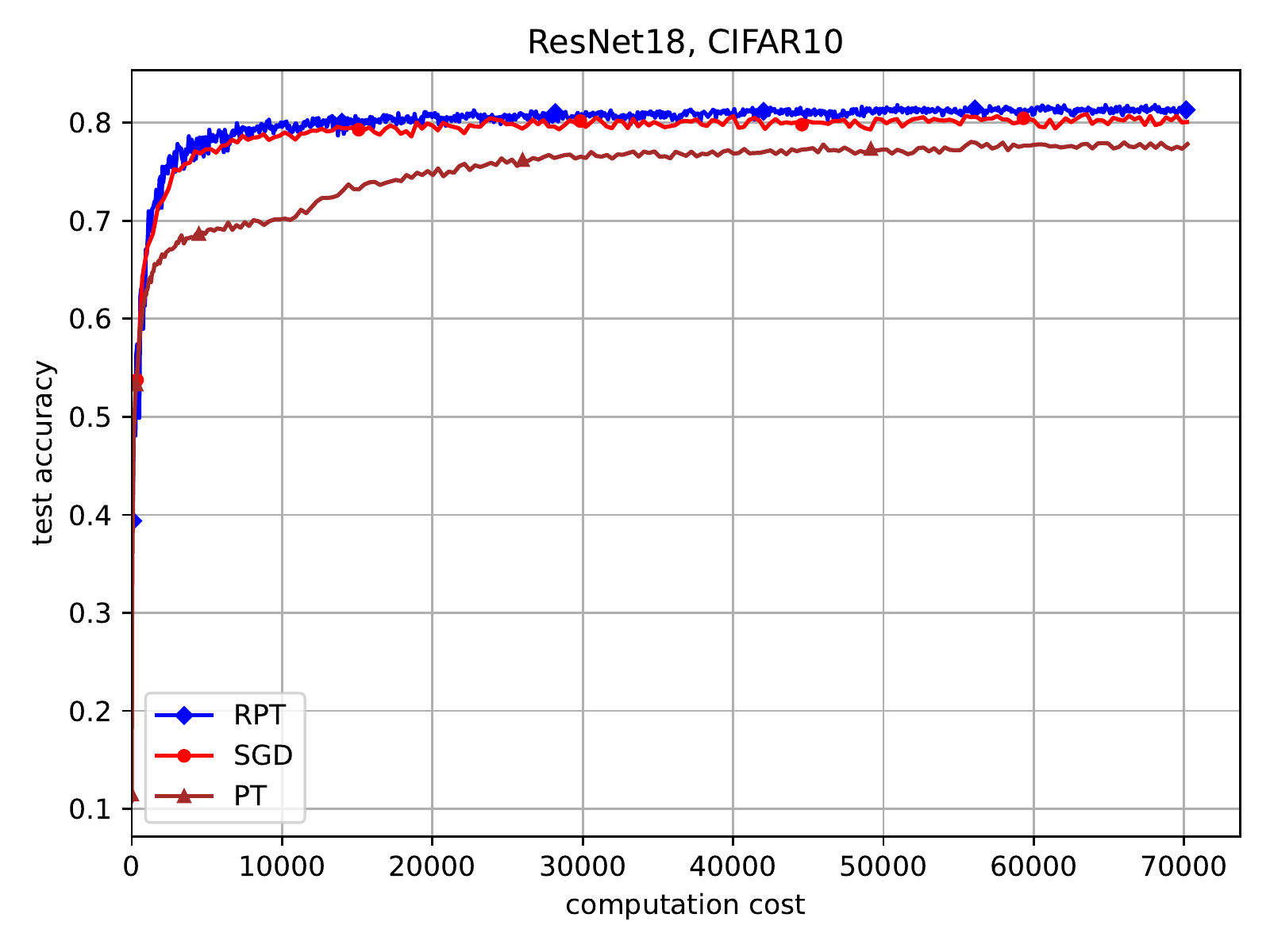}
		\caption{Test accuracy versus computation cost}
		\label{fig:cifar10_test_cost}
	\end{subfigure}
	\\
	\begin{subfigure}[c]{0.5\linewidth}
		\centering		
		\includegraphics[width=\linewidth]{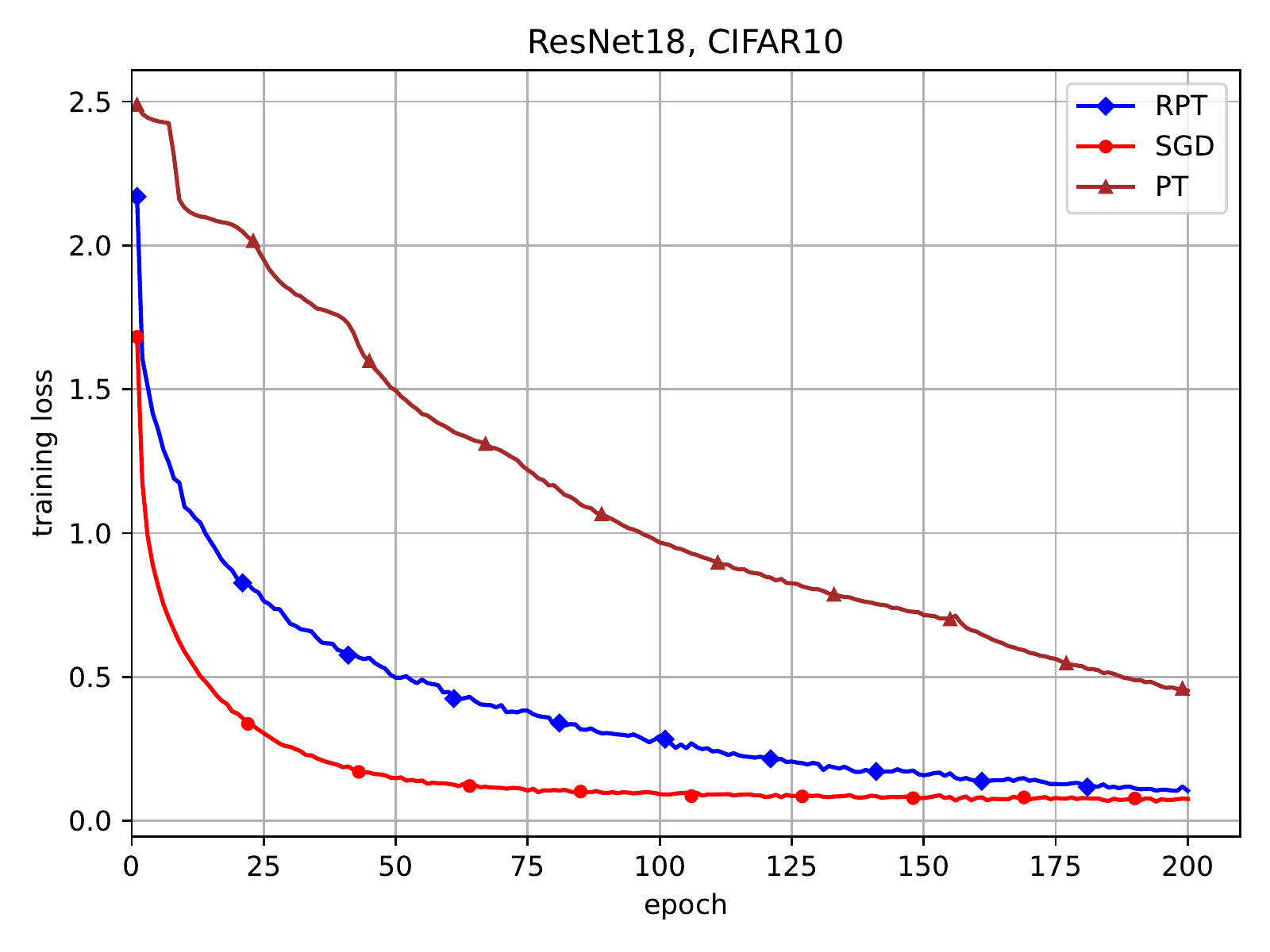}
		\caption{Training loss versus epoch}
		\label{fig:cifar10_train_epoch}
	\end{subfigure}
	\hfill
	\begin{subfigure}[c]{0.5\linewidth}
		\centering		
		\includegraphics[width=\linewidth]{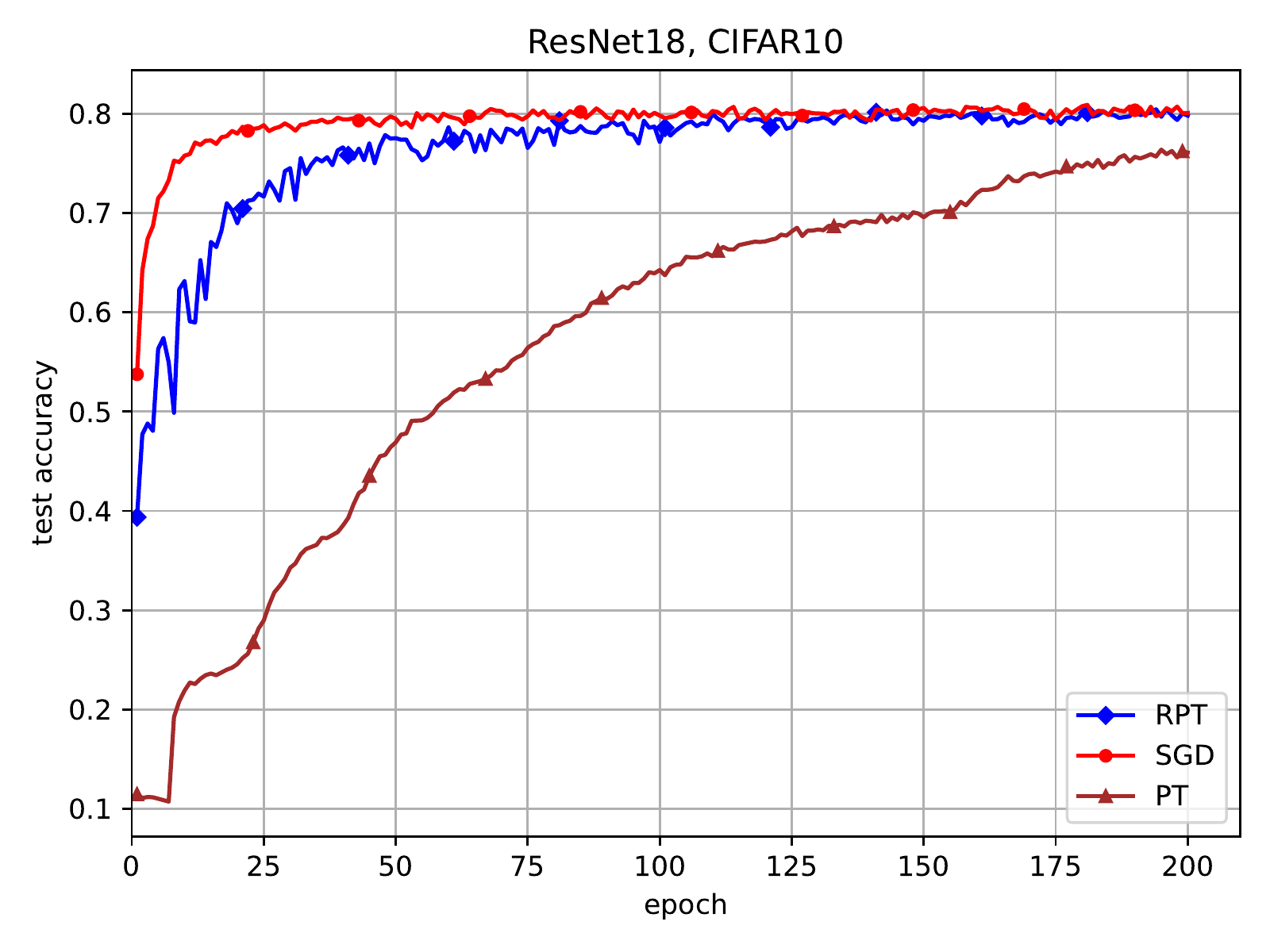}
		\caption{Test accuracy versus epoch}
		\label{fig:cifar10_test_epoch}
	\end{subfigure}%
	\caption{Comparison of Randomized Progressive Training (RPT), Stochastic Gradient Descent (SGD), and Progressive Training (PT) on an image classification task using a ResNet18 model. The block computation cost $c_i$ is proportional to the size of block $i$, with a computation cost of 1 assigned to SGD.}
	\label{fig:cifar10}
\end{figure*}

{\bfseries Results.} The results of our image classification task are summarized in Figure~\ref{fig:cifar10}. While it may appear that RPT converges more slowly and has a lower test accuracy than SGD, based on Figures~\ref{fig:cifar10_train_epoch} and \ref{fig:cifar10_test_epoch}, it is important to note that this comparison is not entirely fair as one RPT step can take significantly less time than one SGD step. To more accurately compare the performance of RPT to its competitors, we have evaluated the performance versus computation cost, as shown in  Figures~\ref{fig:cifar10_train_cost} and \ref{fig:cifar10_test_cost}. From these figures, it is clear that RPT outperforms both competitors.


\bibliography{example_paper}
\bibliographystyle{icml2022}

\newpage
\appendix
\onecolumn

\section{Issues with the theory for ProgFed by \citet{wangProgFed}}
\label{sec:wangEtAl}
We will first briefly outline the notation that \citet{wangProgFed} use in their paper.

We denote by $\mathcal{M}$ the neural network that we aim to train. Put $\mathcal{M}^s$ to be network $\mathcal{M}$ truncated to its first $s$ layers. Let $\mathcal{L}$ denote loss function of interest and put 
$f$ and $f^s$ to be $\mathcal{L}\circ \mathcal{M}$ and $\mathcal{L}\circ \mathcal{M}^s$ respectively. Denote $\mathbf{x}^s$ to be the parameter vector of $\mathcal{M}^s$. Finally, for $s\leq i \leq S$ we put $\mathbf{x}^i_{|E_s}$ and 
$\nabla f^i(\mathbf{x}^i)_{|E_s}$ to be the projections of vectors $\mathbf{x}^i$ and $\nabla f^i(\mathbf{x}^i)$ onto the parameter space of $\mathcal{M}^s$.

The first result by \citet{wangProgFed} proves that
$$
\frac{1}{T}\sum_{t = 0}^{T-1}\alpha_t^2\norm{\nabla f^s(\mathbf{x}_t^s)_{|E_s}} <\varepsilon 
$$
after $\mathcal{O}\left(\frac{\sigma^2}{\epsilon^2}+\frac{1}{\epsilon}\right) \cdot L F_0$ iterations. Here $\mathbf{x}_t^s$ is the $t^\text{th}$ iterate of ProgFed,
$$
\alpha_t = \min\left\{1, \frac{\inp{\nabla f(\mathbf{x}_t)_{|E_s}}{\nabla f^s(\mathbf{x}_t^s)_{|E_s}}}{\norm{\nabla f^s(\mathbf{x}_t^s)}^2}\right\},
$$
$L$ is the smoothness constant of $f$, and $F_0$ is the function suboptimality.
Thus, the numbers $\alpha_t$ depend on the iterates $x_t$. No bounds on $\alpha_t$ are provided. In particular, the theory does not exclude the possibility that $\alpha_t^2$ are extremely close to 0. This renders
the result meaningles.

The second result that \citet{wangProgFed} attempts to bound $\frac{1}{T} \sum_{t=0}^{T-1}\left\|\nabla f\left(\mathbf{x}_t\right)\right\|^2$. Unfortunately, it suffers from a similar issue. The authors claim that
$\frac{1}{T} \sum_{t=0}^{T-1}\left\|\nabla f\left(\mathbf{x}_t\right)\right\|^2$ after 
$$
\mathcal{O}\left(\frac{q^4 \sigma^2}{\epsilon^2}+\frac{q^2}{\epsilon}\right) \cdot L F_0,
$$
where $q:=\max _{t \in[T]}\left(q_t:=\frac{\left\|\nabla f\left(\mathbf{x}_t\right)\right\|}{\alpha_t\left\|\nabla f^s\left(\mathbf{x}_t^s\right)_{\mid E_s}\right\|}\right)$.
\citet{wangProgFed} do not provide any bound on $q$, which does not depend on the iterates. Thus, it is not possible to 
decide \textit{a priori} how many iterations do we need in order to reach the $\varepsilon$-accurate solution, because the theory is dependent on the iterates \emph{produced during the run of the method}. In other words, their second convergence result (in its current state) does not exclude an infinite runtime, so in fact it does not show the convergence at all.

\section{Proofs}
\subsection{Proof of \Cref{lem:key_quantity}}
\keyQuantity*
\begin{proof}
	We start from proving the first inequality. Let $x\in\RR^d$. Then
	\begin{align*}
		x^\top\bP^{1/2}\Exp{\bC\bL\bC}\bP^{1/2}x &= \Exp{\norm{\bL^{1/2}\bC\bP^{1/2}x}^2}\\
		&\leq \lambdamax(\bL)\Exp{\norm{\bC\bP^{1/2}x}^2}\\
		&= \lambdamax(\bL)\Exp{\sum_{i = 1}^n\frac{1}{ p_i}x_i^2\chi_{\{\bC = \bC_i\}}} \\
		&= \lambdamax(\bL)\sum_{i = 1}^nx_i^2 \frac{1}{p_i}p_i \\
		&= \lambdamax(\bL)\norm{x}^2.
	\end{align*}
	Thus, 
	$$
		\lambdamax(\bP^{1/2}\Exp{\bC\bL\bC}\bP^{1/2})\leq \lambdamax(\bL),
	$$
	i.e. $$\lambdamax(\Exp{\bC\bL\bC})\leq \lambdamax(\bP^{-1/2}\bL\bP^{-1/2}).$$
	For the second inequality, similarly,
	\begin{align*}
		x^\top \bP^{-1/2}\bL\bP^{-1/2} x &= \norm{\bL^{1/2}\bP^{-1/2}x}^2 \\
		&\leq \lambdamax(\bL)\norm{\bP^{-1/2}x}^2 \\
		&\leq \lambdamax(\bL)\lambdamax(\bP^{-1})\norm{x}^2.
	\end{align*}
	Thus,
	$$
	\lambdamax(\bP^{-1/2}\bL\bP^{-1/2})\leq \lambdamax(\bL)\lambdamax(\bP^{-1}).
	$$
\end{proof}

\subsection{Proof of \Cref{lem:generalizedSmoothness}}
\generalizedSmoothness*
\begin{proof}
   Fix $\alpha \in \mathbb{R}$ and $x \in \mathbb{R}^d$. Since $x^{\star}$ is a minimizer of $f$ and since $f$ is $\bL$-smooth, we have
		\begin{align*}
			f\left(x^{\star}\right) &\leq f(x-\alpha \bW \nabla f(x))\\
			&\leq \quad f(x)-\alpha\langle\nabla f(x), \bW \nabla f(x)\rangle+\frac{\alpha^2}{2}\|\bW \nabla f(x)\|_\bL^2\\
			&=f(x)-\alpha\langle\nabla f(x), \bW \nabla f(x)\rangle+\frac{\alpha^2}{2}\langle\nabla f(x), \bW \bL \bW \nabla f(x)\rangle
		\end{align*}

	By definition of $A$, we have $\bW^{1 / 2} \bL \bW^{1 / 2} \preceq A \bI_d$, which implies that $\bW \bL \bW \preceq A \bW$. Plugging this estimate into the above inequality, we get
$$
\begin{aligned}
f\left(x^{\star}\right) & \leq f(x)-\alpha\|\nabla f(x)\|_\bW^2+\frac{\alpha^2 A}{2}\|\nabla f(x)\|_\bW^2 \\
&=f(x)-\left(\alpha-\frac{\alpha^2 A}{2}\right)\|\nabla f(x)\|_\bW^2
\end{aligned}
$$
	Since the above inequality holds for any $\alpha \in \mathbb{R}$, it also holds for $\alpha^{\star}=\frac{1}{A}$; this $\alpha$ maximizes the expression $\alpha-\frac{\alpha^2 A}{2}$. It only remains to plug $\alpha^{\star}$ into the above expression and rearrange the inequality.
\end{proof}

\subsection{Proof of \Cref{lem:secondMoment}}
\secondMoment*
\begin{proof}
	\begin{align*}
		\Exp{\norm{g(x)}^2} &= \Exp{\norm{\bC\nabla f(x)}^2} \\
		&= \Exp{\nabla f(x)^\top \bC^2 \nabla f(x)}\\
		&= \nabla f(x)^\top\Exp{\bC^2}\nabla f(x)\\
		&= \nabla f(x)^\top \bP^{-1}\nabla f(x)\\
		&= \norm{\nabla f(x)}_{\bP^{-1}}.
	\end{align*}
	The last inequality follows directly from \Cref{lem:generalizedSmoothness}.
\end{proof}

\subsection{Proof of \Cref{thm:rcdConvex}}
\rcdConvex*
\begin{proof}
   Taking conditional expectations and expanding the brackets yields
	\begin{align*}
		\Exp{\norm{x^{t+1} - x^*}^2|x^t}&= \norm{x^t - x^*}^2 - 2\gamma\Exp{\inp{\bC\nabla f(x^t)}{x^t - x^*}|x^t} + \gamma^2\Exp{\norm{\bC\nabla f(x^t)}^2|x^t} \\
		&\leq \norm{x^t - x^*}^2 - 2\gamma\inp{\nabla f(x^t)}{x^t - x^*} + 2\gamma^2L_\bP\left(f(x^t) - f(x^*)\right)\\
		&\leq \norm{x^t - x^*}^2 - 2\gamma\left(f(x^t) - f(x^*)\right) + \gamma \left(f(x^t) - f(x^*)\right)\\
		&= \norm{x^t - x^*}^2 - \gamma\left(f(x^t) - f(x^*)\right),
	\end{align*}
	where the first inequality follows from \Cref{lem:secondMoment}, and the second one from convexity of $f$ and $0<\gamma\leq 1 /(2L_\bP)$. Rearranging the above inequality, taking the expectation of both sides and using the tower property gives
	\begin{align*}
		\gamma\Exp{f(x^t) - f(x^*)}\leq \Exp{\norm{x^t - x^*}^2 - \norm{x^{t+1} - x^*}^2}.
	\end{align*}
	Thus, by convexity of $f$, we finally arrive at 
	\begin{align*}
		\Exp{f(\bar{x}^{T}) - f(x^*)}&\leq \frac{1}{\gamma(T+1)}\sum_{t = 0}^T\left(f(x^t) - f(x^*)\right)\\
		&\leq \frac{1}{\gamma(T+1)}\left(\norm{x^0 - x^*}^2 - \Exp{\norm{x^{T+1} - x^*}^2}\right)\\
		&\leq \frac{\norm{x^0 - x^*}^2}{\gamma(T+1)}.
	\end{align*}
\end{proof}

\subsection{Proof of \Cref{thm:rcdNonconvex}}
\rcdNonconvex*
Before carrying out the proof of \Cref{thm:rcdNonconvex} we note that in this theorem, one can replace $L_\bP$ by
$$
   L_\bM \eqdef\lambdamax(\Exp{\bC\bL\bC})
$$
which is \emph{at most} $L_\bP$ by \Cref{lem:key_quantity}. However, the constant $L_\bM$ is difficult to interpret qualitatively. For the sake of unified presentation of the RPT analysis in all regimes in question, we therefore chose to replace it by $L_\bP$.
\begin{proof}
   By $\bL$-smoothness we have 
	\begin{align*}
		f(x^{t+1}) &\leq f(x^t) + \inp{\nabla f(x^t)}{x^{t+1} - x^t} + \frac{\norm{x^{t+1} - x^t}_{\bL}^2}{2} \\
		&= f(x^t) - \gamma\inp{\nabla f(x^t)}{\bC\nabla f(x^t)} + \frac{\gamma^2}{2}\norm{\bC\nabla f(x^t)}_{\bL}^2 \\
		&= f(x^t) - \gamma\inp{\nabla f(x^t)}{\bC\nabla f(x^t)} +  \frac{\gamma^2}{2}\inp{\nabla f(x^t)}{\bC\bL\bC\nabla f(x^t)}.
	\end{align*}
	Thus, taking expectation conditional on $x^t$ yields
	\begin{align*}
		\Exp{f(x^{t+1})|x^t}&\leq f(x^t) - \gamma\norm{\nabla f(x^t)}^2 + \frac{\gamma^2}{2}\nabla f(x^t)^\top \Exp{\bC\bL\bC}\nabla f(x^t) \\
		&\leq f(x^t) - \gamma\norm{\nabla f(x^t)}^2 + \frac{\gamma^2}{2}L_{\bM}\norm{\nabla f(x^t)}^2\\
      	&\leq f(x^t) - \gamma\norm{\nabla f(x^t)}^2 + \frac{\gamma^2}{2}L_{\bP}\norm{\nabla f(x^t)}^2,
	\end{align*}
	i.e.
	\begin{align*}
		\left(\gamma - \frac{\gamma^2L_\bP}{2}\right)\norm{\nabla f(x^t)}^2 \leq f(x^t) - \Exp{f(x^{t+1})|x^t}.
	\end{align*}
	By taking expectations of both sides, using tower property and taking the average over $t$, we get
	\begin{align*}
		\frac{\gamma}{2}\min_{t= 0, \dots, T}\norm{\nabla f(x^t)}^2 &\leq \frac{\gamma}{2(T+1)}\sum_{t = 0}^T\norm{\nabla f(x^t)}^2 \\
		&\leq \left(\frac{\gamma}{2} + \frac{\gamma}{2}\left(1 - \gamma L_\bP\right)\right)\frac{1}{T+1}\sum_{t = 0}^T\norm{\nabla f(x^t)}^2\\ 
		&= \left(\gamma - \frac{\gamma^2L_\bP}{2}\right)\frac{1}{T+1}\sum_{i = 0}^T\norm{\nabla f(x^t)}^2 \\
		&\leq \frac{1}{T+1}\Exp{f\left(x^0\right) - f\left(x^{T+1}\right)} \\
		&\leq\frac{1}{T+1}\delta^0.
	\end{align*}
	Rearranging gives
	$$
		\Exp{\norm{\nabla f(\hat{x}^T)}^2}\leq \frac{2\delta^0}{\gamma(T+1)}.
	$$
\end{proof}

\subsection{Proof of \Cref{lem:twoBounds}}
\twoBounds*
\begin{proof}
	The second inequality is just Lemma 1 from \cite{nesterovRCD}. Let us now prove the first inequality. Without loss of generality we may assume that $\max_iL_i = L_1$. Denote $e_i$ to be the $i^\text{th}$ standard unit vector of $\RR^d$. Let $S = \text{Span}(e_1, \dots, e_{d_1})$. Now,
	\begin{align*}
		L_1 &= \sup_{\substack{x\in\RR^{d_1}, \\ \norm{x} = 1}} x^\top \bL_1 x \\
		&= \sup_{\substack{x\in S, \\ \norm{x} = 1}} x^\top \bL x \\
		&\leq \sup_{\substack{x\in \RR^d, \\ \norm{x} = 1}} x^\top \bL x.
	\end{align*}
	This concludes the proof.
\end{proof}

\subsection{Proof of \Cref{lem:upperBound}}
\upperBound*
\begin{proof}
	By Cauchy-Schwarz inequality, 
	\begin{align*}
		\sum_{i = 1}^B\frac{L_{\sigma_i}}{p_i}\cdot\sum_{i = 1}^B p_ic_{\sigma_i} &\geq \left(\sum_{i = 1}^B\sqrt{L_{\sigma_i} / p_i} \cdot \sqrt{p_ic_{\sigma_i}}\right)^2\\
		&= \left(\sum_{i = 1}^B\sqrt{L_i c_i}\right)^2.
	\end{align*}
	Since we are free to choose permutation $\sigma$, we may clearly choose it in such a way that
	$$
		\sqrt{L_{\sigma_1} / c_{{\sigma_1}}} \geq \dots \geq \sqrt{L_{\sigma_B} / c_{{\sigma_B}}}.
	$$
	Plugging in this choice of permutation $\sigma$ and probablities $p_1, \dots, p_B$ verifies that the lower bound from \Cref{lem:upperBound} is attainable.
\end{proof}

\newpage

\section{Discussion about tightness of \Cref{lem:twoBounds}}

\label{sec:tightness}
 Here we consider the 2-block case argue that the maximum eigenvalue $\lambda_{\max}(\bM) $ can be arbitrarily close to $\lambda_{\max}(\bM_1) + \lambda_{\max}(\bM_2)$. Extension of the below result to an arbitrary number of blocks is very straightforward.
\begin{proposition}\label{prop:reachable_upper_bound}
	Let $\bM_1 \in \mathbb{S}_{++}^{d_1}, \bM_2 \in \mathbb{S}_{++}^{d_2} $, and $\bC = \sqrt{\lambda_{\max}(\bM_1) (\lambda_{\max}(\bM_2) - \varepsilon)} \cdot v u^\top$, where $\varepsilon \in (0, \lambda_{\max}(\bM_2))$,  $u \in \RR^{d_1}$ and $v \in \RR^{d_2}$ are unit eigenvectors corresponding to maximum eigenvalues of matrices $\bM_1$ and $\bM_2$, respectively. Then the matrix $\bM = \begin{bmatrix}
		\bM_1 & \bC^\top \\ \bC & \bM_2
	\end{bmatrix}$ is positive definite, and it holds that
	$$
	\lambda_{\max}(\bM) \geq \lambda_{\max}(\bM_1) + \lambda_{\max}(\bM_2) - \varepsilon.
	$$
\end{proposition}
\begin{proof}
	First, let us check the matrix $\bM = \begin{bmatrix}
		\bM_1 & \bC^\top \\ \bC & \bM_2
	\end{bmatrix}$ is positive definite. Since $\bM_1$ is positive definite, it suffices to check that $\bM_2 - \bC \bM_1^{-1} \bC^\top$ is positive definite. The second term can be rewritten as 
	\begin{align*}
	\bC \bM_1^{-1} \bC^\top = \lambda_{\max}(\bM_1) (\lambda_{\max}(\bM_2) - \varepsilon) v \underbrace{u^\top \bM_1^{-1} u}_{= \frac{1}{\lambda_{\max}(\bM_1)}} v^\top =  (\lambda_{\max}(\bM_2) - \varepsilon) v v^\top.
	\end{align*}
	Now, through eigenvalue decomposition of $\bM_2$, we see that eigenvalues of $\bM_2 - \bC \bM_1^{-1} \bC^\top$ are the same as eigenvalues of $\bM_2$ except for the maximum one, which is replaced by $\varepsilon > 0$. Thus, $\bM_2 - \bC \bM_1^{-1} \bC^\top$ is positive definite.

	Let $\bM_1 = \bLambda_{1} \bSigma_{1} \bLambda_{1}^\top$ and $\bM_2 =  \bLambda_{2} \bSigma_{2} \bLambda_{2}^\top$ be eigenvalue decompositions of the matrices, and without the loss of generality assume that the eigenvalues are sorted in a descending order. We also set $a~\eqdef~\sqrt{\lambda_{\max}(\bM_1) (\lambda_{\max}(\bM_2) - \varepsilon)}$. Then 
	\begin{align*}
	\bM' \eqdef \Diag(\bLambda_{1}, \bLambda_{2})^\top \bM  \Diag(\bLambda_{1}, \bLambda_{2}) = \begin{bmatrix}
	\bSigma_A & \Diag(a, 0, \dots, 0)^\top \\
	\Diag(a, 0, \dots, 0) & \bSigma_{2}
	\end{bmatrix}
	\end{align*}
	has the same set of eigenvalues as $\bM$. We now have
	\begin{align*}
	\lambda_{\max}(\bM') &= \max_{\|u'\|^2 + \|v'\|^2 = 1}(u', v')^\top \bM' (u', v') \\
	& = \max_{\norm{u'}^2 + \norm{v'}^2 = 1} u'^\top \bSigma_{1} u' + v'^\top \bSigma_{2} v' + 2u'^\top\Diag(a, 0, \dots, 0)^\top v'\\
	&\geq \max_{x^2 + y^2 = 1}  x^2\lambda_{\max}(\bM_1) +  y^2\lambda_{\max}(\bM_2) + 2  xy\sqrt{\lambda_{\max}(\bM_1) (\lambda_{\max}(\bM_2) - \varepsilon)} \\
	& \geq \max_{x^2 + y^2 = 1}  x^2\lambda_{\max}(\bM_1) +  y^2(\lambda_{\max}(\bM_2) - \varepsilon) + 2 xy\sqrt{\lambda_{\max}(\bM_1) (\lambda_{\max}(\bM_2) - \varepsilon)}  \\
	& = \max_{x^2 + y^2 = 1} \left( x\sqrt{\lambda_{\max}(\bM_1)}  + y\sqrt{\lambda_{\max}(\bM_2) - \varepsilon} \right)^2\\
	& \geq \left( x\sqrt{\lambda_{\max}(\bM_1)} + y\sqrt{\lambda_{\max}(\bM_2) - \varepsilon} \right)^2 \Big{|}_{x = \sqrt{\frac{\lambda_{\max}(\bM_1)}{\lambda_{\max}(\bM_1) + \lambda_{\max}(\bM_2) - \varepsilon}}, y = \sqrt{\frac{\lambda_{\max}(\bM_2) - \varepsilon}{\lambda_{\max}(\bM_1) + \lambda_{\max}(\bM_2) - \varepsilon}}} \\
	& = \lambda_{\max}(\bM_1) + \lambda_{\max}(\bM_2) - \varepsilon.
	\end{align*}
	It remains to note that $\lambda_{\max}(\bM') = \lambda_{\max}(\bM)$, since the sets of eigenvalues coincide.
\end{proof}

\section{Discussion about the unified analysis in nonconvex and convex settings}
\label{sec:Generic}
\subsection{Nonconvex setting}
Recent theoretical breakthrough of \citet{khaledNonConv} offers generic result on the convergence of SGD under very mild assumptions. However, in this section we argue that
the rates promised by this result is suboptimal for the setting of our paper, which is why we derived a separate analysis of RCD in this regime. 

\begin{assumption}[Assumption 2, \citet{khaledNonConv}]
	\label{ass:ABC}
	Let $g(x)$ be an unbiased estimator of $\nabla f(x)$ for all $x\in\RR^d$. We say that $g$ satisfies an \emph{ABC assumption} if
	$$
		\Exp{\norm{g(x)}^2}\leq 2A\left(f(x) - f^{\inf}\right) + B\norm{\nabla f(x)}^2 + C.
	$$
\end{assumption}
for all $x\in\RR^d$.

\begin{theorem}
	\label{thm:non-convex}
	Suppose that $f$ is lower-bounded and $L$-smooth. Let $g$ be an unbiased gradient estimator satisfying \Cref{ass:ABC} with $B=C=0$.
	Let $0<\gamma\leq \frac{1}{\sqrt{LAK}}$. Denote $f(x^0) - f^{\inf}$ by $\delta_0$. Then, the iterates $x^t$ of SGD using the estimator $g$ and stepsize $\gamma$ satisfy
	$$
		\min_{0\leq t \leq T-1}\Exp{\norm{\nabla f(x^t)}^2}\leq \frac{2(1 + L\gamma^2 A)^T\delta_0}{\gamma T}.
	$$
	In particular, for 
	$$
		T\geq \frac{144\delta_0^2LA}{\varepsilon^2}
	$$
	we have 
	$
	\min_{0\leq t \leq T-1}\Exp{\norm{\nabla f(x^t)}^2}\leq\varepsilon.
	$
\end{theorem}

Notice that for nonconvex smooth problems, by Lemma~\ref{lem:secondMoment}, Assumption~\ref{ass:ABC} is satisfied for $A = L_\bP$ and $B=C=0$. Moreover, 
$$
\Exp{\norm{g(x)}^2} = \norm{\nabla f(x)}_{\bP^{-1}}^2 \leq \lambdamax\left(\bP^{-1}\right)\norm{\nabla f(x)}^2,
$$
so \Cref{ass:ABC} is also satisfied for $B = \lambdamax\left(\bP^{-1}\right)$ and $A = C = 0$. There seem to be no other ways in which the RPT estimator can satisfy \Cref{ass:ABC}.

Plugging in both of these cases into \Cref{thm:non-convex} yields suboptimal convergence rates.

Indeed, for $A = L_\bP$ and $B=C=0$ he rate of RPT guaranteed by \Cref{thm:non-convex} is $O\left(1 / \sqrt{T}\right)$, which is suboptimal (we have proved that it is in fact $O(1 / T)$).

For $B = \lambdamax\left(\bP^{-1}\right)$ and $A = C = 0$ the rate is $O(1/T)$ but it is proportional to 
$$
L\lambdamax\left(\bP^{-1}\right) = \lambdamax(\bL)\lambdamax\left(\bP^{-1}\right) \geq L_\bP,
$$
by \Cref{lem:twoBounds}.

We manage to obtain faster rate by exploiting $\bL$-smoothness and the fact that $\Exp{\norm{g(x)}_{\bL}^2}$ can be expressed as a suitable matrix norm.
\subsection{Convex setting}
Applying generic result of \citet{khaledConv} yields somewhat good convergence guarantee ($O(1/T)$ rate is recovered); however, it still suffers from an unnecessary dependency on function suboptimality.
\begin{theorem}[Adapted from \citet{khaledConv}]
	\label{thm:oldConvex}
	Let $f:\RR^d\rightarrow \RR$ be convex, differentiable and lower bounded with global minimizer $x^*$. Suppose that an unbiased gradient estimator $g$ satisfies 
	$$
		\Exp{\norm{g(x)}^2} \leq 2A(f(x) - f(x^*)).
	$$
	Choose stepsize $$\gamma \leq \min \left\{\frac{1}{4A}, \frac{1}{2 L}\right\}.$$ Then the SGD procedure with gradient estimator $g$ and stepsize $\gamma$ satisfies
	$$
	\mathbb{E}\left[f\left(\bar{x}^T\right)-f\left(x^*\right)\right] \leq \frac{2 \gamma\delta^0+\left\|x^0-x^*\right\|^2}{\gamma T},
	$$
	where $\bar{x}^T\eqdef\frac{1}{t}\sum_{i = 0}^{T-1}x^i$ and $\delta^0 \eqdef f(x^0) - f(x^*)$.
\end{theorem}
As discussed previously, the assumption from \Cref{thm:oldConvex} is satisfied for $A = L_\bP \geq L$. Thus, the optimal theoretical stepsize is $\gamma = 1 / (4L_\bP$).
Plugging in this stepsize yields the rate
$$
O\left(\frac{\delta^0 + L_\bP\norm{x^0 - x^*}^2}{T}\right),
$$
whereas we obtained the rate of
$$
O\left(\frac{L_\bP\norm{x^0 - x^*}^2}{T}\right).
$$
\end{document}